\documentclass[11pt]{article}

\usepackage[margin=1in]{geometry}
\usepackage[T1]{fontenc}
\usepackage[utf8]{inputenc}
\usepackage{amsmath,amssymb,amsthm,mathtools,bm}
\usepackage{booktabs,multirow,array}
\usepackage{graphicx}
\usepackage{float}
\usepackage{xcolor}
\usepackage{hyperref}
\usepackage{enumitem}
\usepackage{algorithm}
\usepackage[noend]{algpseudocode}
\usepackage{listings}
\usepackage{tikz}
\usepackage{pgfplots}
\usepackage{caption}

\newcommand{\inkh}{\textsc{InKH}}
\newcommand{\modelonly}{ModelOnly}
\newcommand{\toolagent}{ToolAgent}
\newcommand{\simplemem}{SimpleMem}

\newcommand{\wikiwalk}{WikiWalk}
\newcommand{\khnoinv}{Khnoinv}

\usepackage{authblk}

\usepackage{float}
\usepackage{array}

\usepackage{tikz}
\usetikzlibrary{positioning,fit,calc,arrows.meta}

\tikzset{>=Latex}

\usepackage{tabularx}

\usepackage{graphicx}

\usepackage[margin=1in]{geometry}
\usepackage{authblk}
\usepackage{amsmath, amssymb, amsthm}
\usepackage{bm}
\usepackage{microtype}
\usepackage{booktabs}
\usepackage{array}
\usepackage{multirow}
\usepackage{graphicx}
\usepackage{tikz}

\usepackage{algorithm}
\usepackage{algpseudocode}
\usepackage{listings}
\usepackage{xcolor}
\usepackage[numbers,sort&compress]{natbib}
\usepackage{hyperref}
\hypersetup{
  colorlinks=true,
  linkcolor=blue,
  citecolor=blue,
  urlcolor=blue
}

\usetikzlibrary{arrows.meta,positioning,shapes.geometric,fit,calc}
\pgfplotsset{compat=1.18}
\hypersetup{colorlinks=true,citecolor=blue,linkcolor=blue,urlcolor=blue}

\newtheorem{proposition}{Proposition}

\title{Absorbing Complexity: An Interaction-Native Knowledge Harness for Financial LLM Agents}

\author{

\includegraphics[height=2cm]{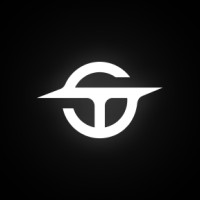}\\
True Trading -- AI\\
\vspace{1.2em} 

\begin{center}
\setlength{\tabcolsep}{0pt}

\begin{minipage}[t]{0.28\textwidth}\centering
\textbf{Ailiya Borjigin}\\
True Trading\\
\texttt{ailiya.borjigin@gmail.com}
\end{minipage}\hspace{12pt}%
\begin{minipage}[t]{0.22\textwidth}\centering
\textbf{Igor Stadnyk}\\
True Trading\\
\texttt{igor@true.trading}
\end{minipage}\hspace{12pt}%
\begin{minipage}[t]{0.22\textwidth}\centering
\textbf{Ben Bilski}\\
True Trading\\
\texttt{ben@true.trading}
\end{minipage}

\vspace{1.2em}

\begin{minipage}[t]{0.26\textwidth}\centering
\textbf{Maksym Chikita}\\
Inc4.net\\
\texttt{m.chikita@inc4.net}
\end{minipage}\hspace{18pt}%
\begin{minipage}[t]{0.26\textwidth}\centering
\textbf{Dmytro Kyrylenko}\\
Inc4.net\\
\texttt{dmitriy.k@inc4.net}
\end{minipage}

\vspace{1.2em}

\begin{minipage}[t]{0.26\textwidth}\centering
\textbf{Sofiia Pidturkina}\\
Inc4.net\\
\texttt{s.pidturkina@inc4.net}
\end{minipage}\hspace{18pt}%
\begin{minipage}[t]{0.26\textwidth}\centering
\textbf{Julia Stadnyk}\\
Inc4.net\\
\texttt{julia.b@inc4.net}
\end{minipage}

\end{center}

}

\date{}

\begin{document}
\maketitle

\begin{abstract}
\textbf{Problem.}
Financial AI adoption is constrained not only by model quality but by \emph{financial cognition friction}: users must repeatedly restate fragmented information, historical judgments, risk preferences, and evolving market assumptions. Existing financial agents remain largely turn-based and workflow-disposable: they answer, retrieve, act, and forget. In high-stakes settings such as market analysis, copy-trading evaluation, and trade preparation, this leads to latency, repeated error, stale-memory reuse, and weak traceability.

\textbf{Approach.}
We propose the \emph{interaction-native knowledge harness} (\inkh), a financial-agent architecture that absorbs complexity into the system rather than transferring it to the user. The architecture combines: (i) an event-stream view of user, market, and tool updates; (ii) a bounded \emph{working context buffer} assembled by passive knowledge injection rather than agent-driven memory search; (iii) a temporal knowledge graph as the low-latency retrieval substrate; (iv) a wiki audit surface for human-readable governance; and (v) background extraction, maturity, decay, and write-time invalidation.

\textbf{Results.}
We report a reproducible controlled benchmark with 24 seeds, 4 rounds, 80 episodes per round, and 6 baselines, for 7{,}680 workflows per baseline and 46{,}080 baseline-conditioned evaluations overall. \inkh achieves mean task quality $0.815$ at $900$ ms mean latency. Relative to an agent-driven wiki-walk memory, \inkh reduces latency by $82.95\%$, token cost by $82.29\%$, and stale-knowledge usage by $96.58\%$, while improving task quality by $0.108$ and decision traceability by $0.461$. Relative to an otherwise similar temporal-graph system without invalidation, \inkh improves quality by $0.050$ and reduces stale-memory usage by $96.58\%$ with comparable serving cost.

\textbf{Impact.}
The results support a central design claim: \emph{adoption happens when complexity is absorbed by the system rather than transferred to the user}. In financial AI, this means continuously transforming interaction traces into structured, persistent, and operational knowledge, while keeping execution safety and auditability explicit.
\end{abstract}

\section{Introduction}

Large language model agents are gradually moving from one-shot question answering toward sustained financial workflows. In practice, these workflows include market analysis, portfolio review, trader evaluation, risk checking, order preparation, and user confirmation. Yet most systems still behave as turn-based assistants: the user asks, the agent retrieves, the agent answers, and most of the useful context disappears afterward. This creates recurring latency, duplicated reasoning, fragile personalization, and repeated rediscovery of the same risks.

Our central argument is that \emph{adoption happens when complexity is absorbed by the system rather than transferred to the user}. In financial AI, users should not have to manually coordinate fragmented information, historical judgments, risk preferences, and changing market assumptions. Instead, a production-grade system should continuously convert interaction traces into structured, persistent, and operational knowledge. We call the associated reduction in user burden \emph{financial cognition friction}.

This paper proposes the \emph{interaction-native knowledge harness} (\inkh), a new architecture for continuous financial cognition. The design draws inspiration from \emph{interaction-native} AI systems that treat collaboration as continuous rather than turn-based \cite{thinkingmachines2026}, from compiled-knowledge approaches such as the LLM Wiki pattern \cite{karpathy2026llmwiki}, and from recent work on scalable agent memory, experiential learning, graph-backed memory, and financial agent benchmarks \cite{chhikara2025mem0,yu2023finmem,shinn2023reflexion,zhao2024expel,wang2023voyager,rasmussen2025zep,islam2023financebench,bigeard2025financeagentbenchmark,choi2025finagentbench,lu2025bizfinbench}.

The paper is also intentionally complementary to recent work on \emph{execution-layer} safety in agentic finance. Borjigin et al.\ argue that in agentic crypto trading, ``execution is the new attack surface'' and propose survivability-aware local executors to enforce non-bypassable last-mile constraints \cite{borjigin2026execution}. Related work by Borjigin and He develops constrained execution and auditable compliance layers for cross-market trade execution \cite{borjigin2025safe}. Taken together, those works address the downstream \emph{action plane}; this paper addresses the upstream \emph{cognition plane}. Safe financial agents require both.

The contributions of this paper are fourfold:

\begin{enumerate}[leftmargin=1.5em]
    \item We propose \inkh, an interaction-native architecture for continuous financial cognition built from passive knowledge injection, a bounded working context buffer, temporal graph memory, and a wiki audit surface.
    \item We formalize the architecture with explicit state, knowledge objects, retrieval utility, injection, decay, invalidation, maturity transition, and governance constraints.
    \item We provide algorithms for passive injection, background extraction, and maintenance, together with implementation guidance for graph+wiki systems, entity matching, upsert semantics, and latency budgeting.
    \item We report a reproducible experimental study on a controlled financial benchmark and show that \inkh materially improves quality, latency, stale-memory suppression, repeated error reduction, and traceability relative to memory and non-memory baselines.
\end{enumerate}

\section{Literature Review}

\subsection*{Interaction-native collaboration}

A recent technical report from Thinking Machines Lab argues that current AI systems suffer from a communication bottleneck because they alternate between discrete input and output phases rather than collaborating continuously \cite{thinkingmachines2026}. Their proposed \emph{interaction models} natively process overlapping streams of audio, video, and text, and use a real-time foreground model together with asynchronous background reasoning. This paper adopts the architectural lesson rather than the multimodal training recipe: financial cognition should be \emph{continuous and stateful}, even when operationalized over text and event streams rather than full-duplex speech or video.

\subsection*{Compiled knowledge, memory, and graph-backed retrieval}

Karpathy's LLM Wiki pattern proposes an appealing separation between immutable raw sources, LLM-maintained wiki pages, and explicit ingest/query/lint operations \cite{karpathy2026llmwiki}. In parallel, the memory literature has advanced rapidly. MemGPT frames multi-tier memory management as an operating-system problem \cite{packer2023memgpt}. Mem0 emphasizes production-ready long-term memory and reports large latency and token improvements over full-context baselines \cite{chhikara2025mem0}. A-MEM pushes further toward dynamically organized ``agentic memory'' \cite{xu2025amem}. Zep/Graphiti introduces a temporal knowledge graph memory layer with explicit validity windows and relationship-aware retrieval \cite{rasmussen2025zep}, while a recent survey synthesizes graph-based memory around extraction, storage, retrieval, and evolution \cite{yang2026graphsurvey}. \inkh{} is positioned at a different layer than Graphiti. Graphiti provides temporal validity windows and relationship-aware retrieval at the storage layer, whereas \inkh{} operates at the orchestration layer: retrieval is passive and system-injected rather than agent-requested, and invalidation is performed at write time during background extraction rather than only at query time. \inkh{} further adds a governance layer that gates which knowledge may influence which financial actions based on maturity and action risk. \inkh{} is therefore not a replacement for Graphiti; a production implementation could use Graphiti or a similar temporal graph system as the underlying substrate. Our position is that finance needs both \emph{compiled knowledge} and \emph{governed graph retrieval}: the graph should serve online retrieval, while the wiki remains the audit surface.

\subsection*{Experience accumulation and self-improvement}

Several influential papers show that agents can improve without finetuning by learning from their past trajectories. Reflexion introduces verbal reinforcement and episodic reflection \cite{shinn2023reflexion}. ExpeL extracts reusable knowledge from a set of training experiences \cite{zhao2024expel}. Voyager accumulates reusable skills through an expanding code library \cite{wang2023voyager}. Generative Agents and Self-Refine likewise demonstrate the importance of reflection, memory, and iterative improvement in interactive systems \cite{park2023generative,madaan2023selfrefine}. Self-RAG adds adaptive retrieval and self-critique \cite{asai2023selfrag}. These works support the broader claim that past experience should become future capability, but they do not address the governance and temporal invalidation demands of financial reasoning.

\subsection*{Financial agents, retrieval, and benchmarks}

Finance-specific agent research remains comparatively sparse. FinMem is a notable early effort that builds a layered-memory trading agent with profiling and decision modules \cite{yu2023finmem}. Benchmarking work has accelerated more recently. FinanceBench shows that open-book financial QA remains difficult even for strong retrieval-augmented systems \cite{islam2023financebench}. Finance Agent Benchmark tests real-world research tasks involving SEC filings and expert-authored questions \cite{bigeard2025financeagentbenchmark}. BizFinBench extends benchmark coverage to business-driven financial tasks \cite{lu2025bizfinbench}. FinAgentBench targets \emph{agentic retrieval} in financial question answering \cite{choi2025finagentbench}. These studies clarify what current systems still miss: not only domain knowledge, but robust multi-step retrieval and repeated, governed adaptation.

\subsection*{Execution, governance, and financial agent stacks}

Recent financial-agent papers by Borjigin and collaborators provide an important complementary perspective. \emph{Execution Is the New Attack Surface} formalizes survivability-aware execution middleware for agentic crypto trading \cite{borjigin2026execution}. \emph{Safe and Compliant Cross-Market Trade Execution} adds constrained reinforcement learning, action shielding, and auditable compliance \cite{borjigin2025safe}. Our paper extends that trajectory upstream: before actions are made survivable, the agent must maintain the right evolving financial state.

\subsection*{Classical retrieval and tool-using agents}

The proposed design also builds on standard retrieval and tool-use literature. Retrieval-Augmented Generation (RAG) formalized non-parametric memory augmentation \cite{lewis2020rag}, RETRO showed retrieval at enormous scale \cite{borgeaud2021retro}, ReAct unified reasoning and acting \cite{yao2022react}, and Toolformer demonstrated self-supervised tool use \cite{schick2023toolformer}. \inkh differs from these approaches in one crucial aspect: \emph{retrieval is not solely a query-time decision made by the agent}. Instead, relevant knowledge is injected into a bounded working state before the next reasoning step.

\section{Problem Formulation}

We model a financial agent as operating over an event stream
\[
e_t \in \mathcal{E},
\]
where events may be user turns, tool observations, market updates, portfolio changes, or internally generated risk signals.

A workflow episode is
\[
w_i = \big(e_{t_i:t_i+\ell_i}, a_{i,1:H_i}, o_{i,1:H_i}, y_i\big),
\]
where $a_{i,h}$ is an internal action such as retrieval or tool invocation, $o_{i,h}$ is the resulting observation, and $y_i$ is the user-visible output.

\subsection*{State and knowledge objects}

At time $t$, the agent maintains state
\[
S_t = (U_t, M_t, R_t, X_t, G_t),
\]
where $U_t$ is user state, $M_t$ is market state, $R_t$ is risk state, $X_t$ is workflow or execution state, and $G_t$ is the temporal knowledge graph.

Each knowledge object is represented as
\[
k = (\tau,\sigma,\phi,\omega,c,\mu,\rho,t_f,t_l,t_i).
\]
Here $\tau$ is the type, $\sigma$ the scope, $\phi$ the content, $\omega$ the evidence and provenance, $c$ the confidence, $\mu$ the maturity, $\rho$ the regime tag, $t_f$ the first-seen time, $t_l$ the last-validated time, and $t_i$ the invalidation time if the item has been superseded.

\subsection*{Working context buffer}

Let
\[
(V_t,\iota_t,\chi_t) = \mathrm{Detect}(e_t,S_t)
\]
denote detected active entities, intent, and risk class.

Candidate knowledge is drawn from the graph neighborhood
\[
\mathcal{C}_t = \mathcal{N}_h(V_t;G_t)\setminus \{k : t_i(k)\le t\},
\]
where $\mathcal{N}_h$ is the $h$-hop graph neighborhood and invalidated items are excluded.

Each candidate receives utility
\[
s_t(k)
=
\alpha_{\mathrm{rel}} \mathrm{Rel}(k,e_t)
+
\alpha_{\mathrm{str}} \mathrm{Struct}(k,V_t)
+
\alpha_{\mathrm{mat}} \mathrm{Mat}(\mu(k))
+
\alpha_{\mathrm{fresh}} \mathrm{Fresh}(k,t)
+
\alpha_{\mathrm{reg}} \mathrm{Regime}(k,M_t)
+
\alpha_{\mathrm{trust}} \mathrm{Trust}(k)
-
\alpha_{\mathrm{noise}} \mathrm{Noise}(k).
\]

The \emph{injection operator} under token budget $B$ is
\[
I_t
=
\mathrm{Compress}\!\Big(
\mathrm{TopB}\{s_t(k)\; : \; k \in \mathcal{C}_t \cap \mathcal{A}_t\}
\Big),
\]
where $\mathcal{A}_t$ is the set of governance-admissible items.

The \emph{working context buffer} is then
\[
C_t = \mathrm{Fuse}(U_t,M_t,R_t,X_t,I_t),
\qquad |C_t| \le B.
\]
Here $C_t$ denotes the bounded \emph{context} assembled for the next reasoning step.

This object is the online working state given to the agent for the next step. The conceptual difference from agent-driven retrieval is simple: the system \emph{prepares} the context instead of forcing the model to search for it.

\subsection*{Objective}

The agent is optimized not only for task performance but also for efficient, safe, and knowledge-compounding operation:
\[
J_t
=
\mathbb{E}\!\left[
Q(y_t)
-
\lambda_c \mathrm{Cost}_t
-
\lambda_r \mathrm{ActRisk}_t
+
\lambda_k \mathrm{KnowGain}_t
\right].
\]
This objective formalizes the intended tradeoff between answer quality, serving cost, action risk, and future knowledge benefit.

\subsection*{Knowledge update, decay, invalidation, and maturity}

After a completed workflow, background extraction produces candidate knowledge
\[
Z_i = \mathrm{Extract}(w_i),
\]
and the graph is updated by
\[
G_{t+1} = \mathrm{Upsert}(G_t, Z_i).
\]

Continuous decay is modeled by
\[
d_t(k)
=
\exp\!\big(-\lambda_{\tau(k)}(t-t_l(k))\big)
\exp\!\big(-\gamma \, \mathrm{dist}(\rho(k),\hat{\rho}_t)\big),
\]
where $\lambda_{\tau(k)}$ is type-dependent and $\hat{\rho}_t$ is the current inferred regime.

Effective confidence becomes
\[
c_t^{\mathrm{eff}}(k)
=
c(k)\, d_t(k)\, \mathbb{1}[t_i(k)=\varnothing].
\]

Contradiction-triggered invalidation is handled explicitly: if a new item $k'$ contradicts prior item $k$ with score $\Gamma(k',k)$, then
\[
\Gamma(k',k) > \delta
\quad \Longrightarrow \quad
t_i(k)\leftarrow t.
\]

Maturity evolves as
\[
\mu_{t+1}(k)
=
\Psi\!\big(
\mu_t(k), \nu_t(k), \upsilon_t(k), q_t(k), h_t(k)
\big),
\]
where $\nu_t$ is reuse count, $\upsilon_t$ is validation evidence, $q_t$ is downstream utility attribution, and $h_t$ is human review.

\subsection*{Governance constraints}

A knowledge item may influence a financial action only if governance allows it:
\[
\mathrm{Allow}(a,k)
=
\mathbb{1}\!\left[
c_t^{\mathrm{eff}}(k)\ge \epsilon
\;\land\;
\mu(k)\ge \theta(\mathrm{risk}(a))
\;\land\;
\sigma(k)\in \mathcal{O}(u)
\right].
\]
Here $\mathcal{O}(u)$ is the permitted knowledge overlay for user $u$ and $\theta(\cdot)$ is a monotone threshold increasing with action risk.

\subsection*{Theoretical propositions}

\begin{proposition}[Passive injection versus agent-driven retrieval]
Assume that solving a task without reusable knowledge has expected cost $C_0$. Let related prior knowledge yield expected savings $\Delta>0$ with probability $p$. Let passive injection cost $c_p$ with irrelevant-context penalty $\eta_p$, and let agent-driven retrieval incur additional planning cost $c_\ell$ and penalty $\eta_a$. Then
\[
\mathbb{E}[C_{\mathrm{passive}}]
=
C_0 - p\Delta + c_p + \eta_p,
\]
\[
\mathbb{E}[C_{\mathrm{agent}}]
=
C_0 - p\Delta + c_p + c_\ell + \eta_a.
\]
Hence passive injection is cheaper whenever
\[
c_\ell > \eta_p - \eta_a.
\]
\end{proposition}

\begin{proof}
Subtract the two expectations:
\[
\mathbb{E}[C_{\mathrm{agent}}]-\mathbb{E}[C_{\mathrm{passive}}]
=
c_\ell + \eta_a - \eta_p.
\]
If the extra planning or wiki-walk overhead exceeds the marginal irrelevant-context penalty of passive injection, passive injection is strictly cheaper in expectation.
\end{proof}

\begin{proposition}[Governance reduces noise amplification]
Suppose noisy memory items reproduce with effective branching factor $\beta>0$, and governance suppresses a fraction $g\in[0,1]$ of noisy items before reuse. If $N_t$ is the noisy memory mass at time $t$, then
\[
\mathbb{E}[N_{t+1} \mid N_t]
\le
\zeta_t + \beta(1-g)N_t,
\]
where $\zeta_t$ is fresh source noise. If $\beta(1-g)<1$, noisy-memory growth is subcritical and remains bounded in expectation.
\end{proposition}

\begin{proof}
The recurrence defines a linear branching process with ratio $\beta(1-g)$. If the ratio is less than one, the corresponding geometric series converges and expected noisy-memory mass remains bounded.
\end{proof}

\begin{proposition}[Maturity gating for high-risk actions]
Let the correctness probability of maturity state $\mu$ be $\pi_\mu$, and let action $a$ have benefit $B_a$ when based on correct knowledge and loss $L_a$ scaled by risk multiplier $\lambda_a$ when based on incorrect knowledge. Then
\[
EU(a,\mu)
=
\pi_\mu B_a - (1-\pi_\mu)\lambda_a L_a.
\]
Using maturity level $\mu$ for action $a$ is rational only if
\[
\pi_\mu \ge \frac{\lambda_a L_a}{B_a + \lambda_a L_a}.
\]
Because the right-hand side increases with $\lambda_a$, minimum acceptable maturity must rise with action risk.
\end{proposition}

\begin{proof}
Rearrange the condition $EU(a,\mu)\ge 0$. Monotonicity in $\lambda_a$ is immediate.
\end{proof}

\section{Interaction-Native Knowledge Harness}
\label{sec:arch}

\subsection*{Architecture}

Figure~\ref{fig:architecture} shows the proposed system. The online path is interaction-native: every incoming event updates detected entities, retrieves a compact, governed neighborhood from the temporal graph, and injects it into the working context buffer before the main agent step. The offline path is knowledge-native: completed workflows are extracted, upserted, invalidated when contradicted, and summarized into a wiki audit surface.

\begin{figure}[t]
\centering
\includegraphics[width=0.98\linewidth]{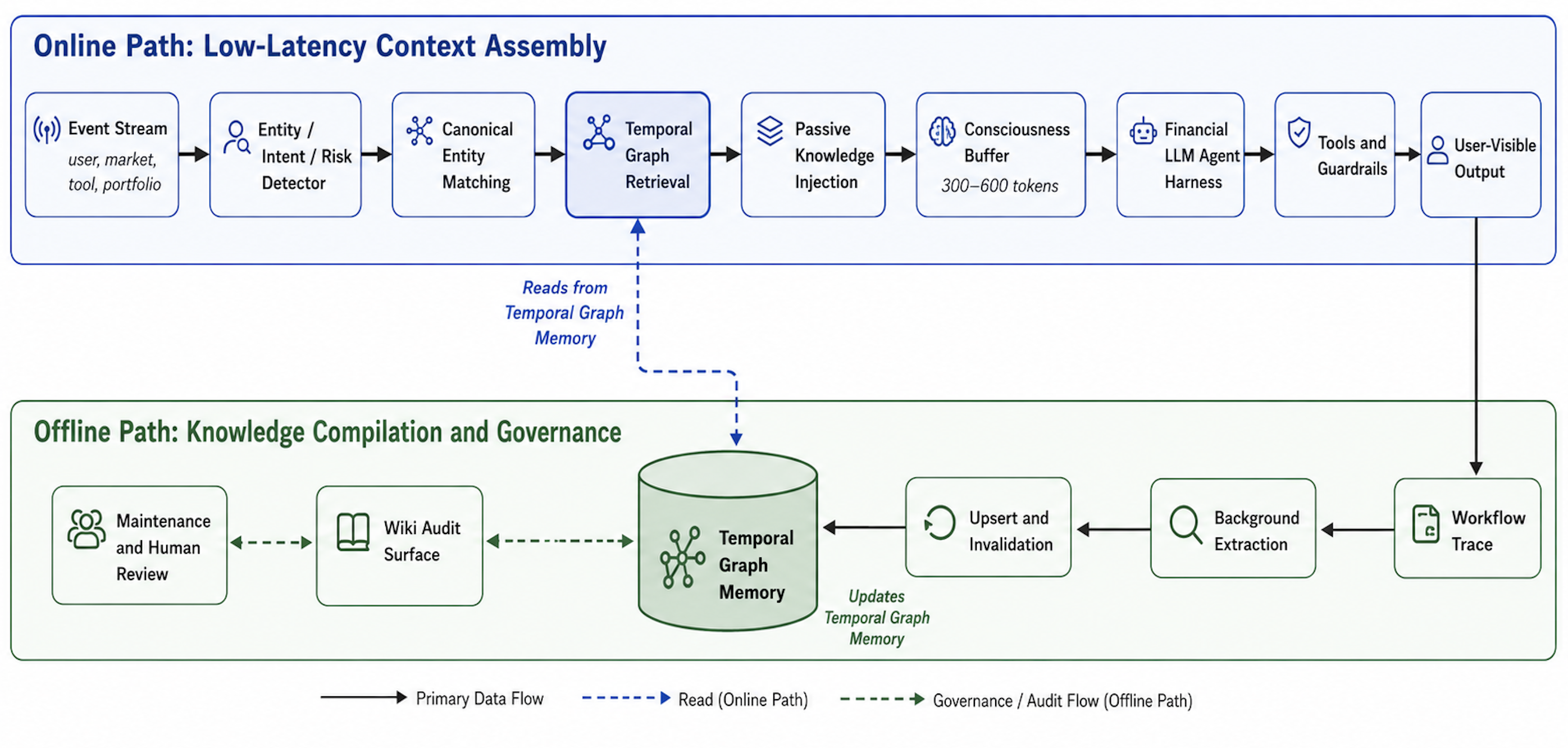}
\caption{Interaction-native knowledge harness. The online path assembles governed low-latency context before the main agent step, while the offline path compiles workflow traces into persistent, auditable long-term knowledge.}
\label{fig:architecture}
\end{figure}

\subsection*{Graph as retrieval substrate, wiki as audit surface}

A key practical design choice is to separate online retrieval from human-readable synthesis. The \emph{temporal graph} is the retrieval substrate: it stores canonical entities, typed relations, provenance, last-validated time, and invalidation. The \emph{wiki} is the audit surface: it stores readable asset pages, trader pages, strategies, risk notes, and maintenance logs. This division keeps online retrieval compact and cheap while preserving inspectability and human review \cite{karpathy2026llmwiki,rasmussen2025zep}.

\subsection*{Baseline comparison}

Table~\ref{tab:baseline-components} contrasts the baselines used in this paper.

\begin{table}[t]
\centering
\footnotesize
\caption{Architectural components present in each baseline.}
\label{tab:baseline-components}
\begin{tabular}{lccccccc}
\toprule
Baseline 
& Tools 
& \shortstack{Persistent\\State} 
& \shortstack{Compiled\\Wiki} 
& \shortstack{Temporal\\Graph} 
& \shortstack{Passive\\Injection} 
& Invalidation 
& \shortstack{Audit\\Trace} \\
\midrule
\textsc{Model-Only}    & $\times$     & $\times$     & $\times$     & $\times$     & $\times$     & $\times$     & Low \\
\textsc{Tool-Agent}    & $\checkmark$ & $\times$     & $\times$     & $\times$     & $\times$     & $\times$     & Low \\
\textsc{Simple-Memory} & $\checkmark$ & $\checkmark$ & $\times$     & $\times$     & $\times$     & $\times$     & Medium \\
\textsc{Wiki-Walk}     & $\checkmark$ & $\checkmark$ & $\checkmark$ & $\times$     & $\times$     & $\times$     & Medium \\
\textsc{KH-NoInv}      & $\checkmark$ & $\checkmark$ & $\checkmark$ & $\checkmark$ & $\checkmark$ & $\times$     & High \\
\textsc{INKH}          & $\checkmark$ & $\checkmark$ & $\checkmark$ & $\checkmark$ & $\checkmark$ & $\checkmark$ & Very high \\
\bottomrule
\end{tabular}
\end{table}

\subsection*{Algorithms}

\begin{algorithm}[H]
\caption{Passive Injection}
\label{alg:inject}
\begin{algorithmic}[1]
\Require event $e_t$, state $S_t$, graph $G_t$, token budget $B$, hop radius $h$
\State $(V_t,\iota_t,\chi_t) \gets \mathrm{Detect}(e_t,S_t)$
\State $V_t \gets \mathrm{Canonicalize}(V_t)$
\State $\mathcal{C}_t \gets \mathcal{N}_h(V_t;G_t)$
\State $\mathcal{C}_t \gets \mathrm{RemoveInvalidated}(\mathcal{C}_t)$
\State $\mathcal{C}_t \gets \mathrm{FilterByGovernance}(\mathcal{C}_t,\chi_t)$
\ForAll{$k \in \mathcal{C}_t$}
    \State compute $s_t(k)$
\EndFor
\State $I_t \gets \mathrm{Compress}(\mathrm{TopB}(\mathcal{C}_t,s_t,B))$
\State \Return $C_t \gets \mathrm{Fuse}(S_t,I_t)$
\end{algorithmic}
\end{algorithm}

\begin{algorithm}[H]
\caption{Background Extraction and Upsert}
\label{alg:extract}
\begin{algorithmic}[1]
\Require completed workflow $w_i$, graph $G_t$
\State $Z_i \gets \mathrm{ExtractCandidateKnowledge}(w_i)$
\State $Z_i \gets \mathrm{AttachEvidenceAndTrust}(Z_i)$
\State $Z_i \gets \mathrm{MatchOrCreateCanonicalEntities}(Z_i,G_t)$
\State $\mathcal{F}_i \gets \mathrm{DetectContradictions}(Z_i,G_t)$
\State $\mathrm{MarkInvalidated}(\mathcal{F}_i)$
\State $G_{t+1} \gets \mathrm{Upsert}(G_t,Z_i)$
\State $\mathrm{UpdateWikiAuditPages}(Z_i,\mathcal{F}_i)$
\State \Return $G_{t+1}$
\end{algorithmic}
\end{algorithm}

\begin{algorithm}[H]
\caption{Maintenance Tick}
\label{alg:maint}
\begin{algorithmic}[1]
\Require temporal graph $G_t$
\State probe $\gets \mathrm{SampleProbeType}()$
\State $\mathrm{RunStalenessSweep}(G_t,\text{probe})$
\State $\mathrm{RunLinkDiscoveryAndMergeChecks}(G_t,\text{probe})$
\State $\mathrm{RecomputeConfidenceAndMaturity}(G_t)$
\State $\mathrm{AutoExecuteLowRiskFixes}(G_t)$
\State $\mathrm{QueueHighImpactChangesForHumanReview}(G_t)$
\State \Return cleaned graph and maintenance log
\end{algorithmic}
\end{algorithm}

\subsection*{Implementation notes}

A production implementation should satisfy four engineering constraints.

\emph{First}, retrieval must be algorithmic and budgeted. For the common case, the system should inject a compact context block without requiring an LLM-driven wiki walk. \emph{Second}, entity matching must collapse aliases such as \texttt{BTC}, \texttt{Bitcoin}, and \texttt{BTCUSDT} into stable canonical entities. \emph{Third}, invalidation must happen at write time when new evidence supersedes older knowledge. \emph{Fourth}, the wiki remains essential, but chiefly as an audit surface rather than an online retrieval primitive.

\section{Experimental Design}

\subsection*{Reported benchmark and public-data extension}

This paper reports results for a \emph{controlled synthetic benchmark} and also specifies, but does not execute, a \emph{public-data replay extension}. Table~\ref{tab:datasets} distinguishes the two.

\begin{table}[t]
\centering
\footnotesize
\caption{Evaluation stages. All reported quantitative results in this paper come from Stage A.}
\label{tab:datasets}
\begin{tabular}{p{1.2cm} p{3.1cm} p{5.5cm} p{3.7cm}}
\toprule
Stage & Status & Data sources & Purpose \\
\midrule
A & Reported & Controlled synthetic workflows, simulated user dialogues, latent preference signals, regime shocks, and stale-knowledge traps & Isolate architectural effects on latency, token cost, reuse, invalidation, and traceability \\
B & Specified only & Public market and filing replay built from FRED, SEC EDGAR, and Binance public market-data interfaces \cite{fred2026api,sec2025edgar,binance2026api} plus simulated user dialogues & Historical replay and human-scored extension for future work \\
\bottomrule
\end{tabular}
\end{table}

The implemented artifact uses 24 seeds, 4 rounds, and 80 episodes per round. This yields 7{,}680 workflows per baseline and 46{,}080 baseline-conditioned evaluations in total. The four task families are \emph{market analysis}, \emph{portfolio review}, \emph{copy-trading evaluation}, and \emph{trade preparation}. Round~1 is a cold start. Round~2 introduces user preference signals. Round~3 injects regime or protocol shocks. Round~4 measures post-shock reuse.

\subsection*{Assumptions}

The reported suite is architecture-level rather than vendor-level. We therefore make four explicit assumptions:

\begin{enumerate}[leftmargin=1.5em]
    \item \textbf{A1 (Model abstraction).} Baselines are modeled by characteristic token budgets, retrieval behavior, and latency distributions rather than tied to one proprietary API.
    \item \textbf{A2 (Cost abstraction).} Token cost is modeled as \$3.00 per million tokens plus \$0.002 per tool call.
    \item \textbf{A3 (Latent ground truth).} Task quality, stale-memory violations, and traceability are assessed against simulator-defined gold requirements.
    \item \textbf{A4 (Scope).} Reported results validate system behavior, not live trading profitability.
\end{enumerate}

\subsection*{Baselines}

We compare six systems: \modelonly, \toolagent, \simplemem, \wikiwalk, \khnoinv, and the full \inkh. The key comparison is between \wikiwalk\ and \inkh: both have compiled persistent knowledge, but only \inkh\ uses passive injection and write-time invalidation.

\subsection*{Metrics}

The main metrics are:
\begin{align*}
\text{Context precision} &= \frac{\text{gold hits}}{\max(1,\text{retrieved items})},\\
\text{Repeated error reduction} &= \frac{(1-Q_1) - (1-Q_T)}{1-Q_1},\\
\text{Cost efficiency} &= \frac{Q}{\text{tokens}/1000}.
\end{align*}
We also report latency, total tokens, stale-knowledge usage, decision traceability, and estimated serving cost.

\subsection*{Statistical testing}

All confidence intervals are 95\% bootstrap intervals over seed-level means with 3{,}000 resamples. Pairwise comparisons against \inkh\ use paired Wilcoxon signed-rank tests over the 24 seed-level means.

\section{Results}
\label{sec:results}

\subsection*{Main quantitative results}

Table~\ref{tab:main-results} presents the central results. The full \inkh\ baseline has the highest task quality and traceability among all systems, while also delivering lower latency than all nontrivial retrieval baselines.

\begin{table}[t]
\centering
\footnotesize
\caption{Main results over 7{,}680 workflows per baseline.}
\label{tab:main-results}
\resizebox{\linewidth}{!}{
\begin{tabular}{lrrrrrrr}
\toprule
Baseline & Quality & Latency (ms) & Tokens & Context precision & Stale usage & Traceability & Est.\ cost (\$) \\
\midrule
\inkh            & \textbf{0.815} & \textbf{900.2} & 1540.3 & \textbf{0.329} & \textbf{0.009} & \textbf{0.999} & 0.0086 \\
\khnoinv         & 0.765 & 960.0 & 1550.2 & \textbf{0.329} & 0.271 & 0.928 & 0.0087 \\
\wikiwalk        & 0.707 & 5281.1 & 8697.3 & 0.163 & 0.271 & 0.538 & 0.0301 \\
\simplemem       & 0.657 & 1931.1 & 3039.7 & 0.246 & 0.233 & 0.501 & 0.0131 \\
\toolagent       & 0.540 & 1351.6 & 1279.5 & 0.000 & 0.000 & 0.320 & 0.0078 \\
\modelonly       & 0.440 & 654.5  & \textbf{880.0} & 0.000 & 0.000 & 0.080 & \textbf{0.0026} \\
\bottomrule
\end{tabular}
}
\end{table}

Relative to \wikiwalk, \inkh\ reduces latency by $82.95\%$ and token load by $82.29\%$, while improving quality by $0.108$, reducing stale-memory usage by $96.58\%$, and raising decision traceability by $0.461$. Relative to \simplemem, \inkh\ improves quality by $0.157$ and lowers stale usage by $96.02\%$. Relative to \khnoinv, \inkh\ yields a $0.050$ increase in quality and a $96.58\%$ reduction in stale-memory usage with nearly identical token budget.

Figure~\ref{fig:frontier} visualizes the quality-latency frontier.

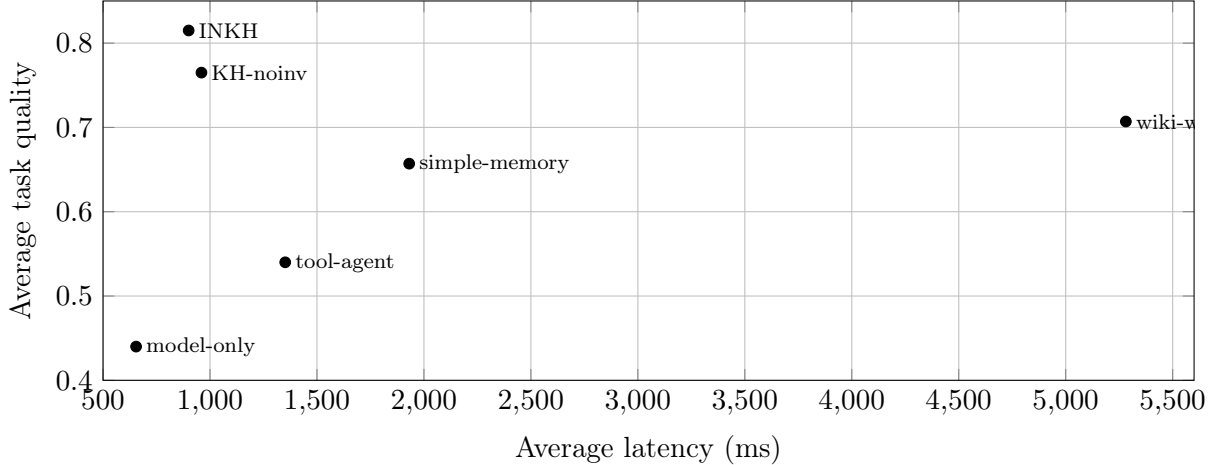
\begin{figure}[t]
\centering
\begin{tikzpicture}
\begin{axis}[
    width=0.97\linewidth,
    height=6.6cm,
    xlabel={Average latency (ms)},
    ylabel={Average task quality},
    xmin=500, xmax=5600,
    ymin=0.40, ymax=0.85,
    grid=major
]
\addplot[only marks, mark=*, mark size=2pt] coordinates {
    (654.5,0.440)
    (1351.6,0.540)
    (1931.1,0.657)
    (5281.1,0.707)
    (960.0,0.765)
    (900.2,0.815)
};
\node[anchor=west,font=\scriptsize] at (axis cs:654.5,0.440) {model-only};
\node[anchor=west,font=\scriptsize] at (axis cs:1351.6,0.540) {tool-agent};
\node[anchor=west,font=\scriptsize] at (axis cs:1931.1,0.657) {simple-memory};
\node[anchor=west,font=\scriptsize] at (axis cs:5281.1,0.707) {wiki-walk};
\node[anchor=west,font=\scriptsize] at (axis cs:960.0,0.765) {KH-noinv};
\node[anchor=west,font=\scriptsize] at (axis cs:900.2,0.815) {INKH};
\end{axis}
\end{tikzpicture}
\caption{Quality-latency frontier. \inkh\ occupies the best region among persistent-memory baselines.}
\label{fig:frontier}
\end{figure}

\subsection*{Significance against memory baselines}

Table~\ref{tab:significance} reports paired comparisons against the three most relevant memory baselines. All gains are statistically significant.

\begin{table}[t]
\centering
\footnotesize
\caption{Paired comparison of \inkh\ against memory baselines. Differences are \inkh\ minus comparator.}
\label{tab:significance}
\resizebox{\linewidth}{!}{
\begin{tabular}{lccccc}
\toprule
Comparator & $\Delta$Quality [95\% CI] & $\Delta$Latency [95\% CI] & $\Delta$Stale [95\% CI] & $\Delta$Trace [95\% CI] & $p_{\max}$ \\
\midrule
\khnoinv & $0.050\ [0.048,\,0.051]$ & $-59.8\ [-60.9,\,-58.7]$ & $-0.261\ [-0.272,\,-0.250]$ & $0.071\ [0.070,\,0.072]$ & $<1.8\times 10^{-5}$ \\
\simplemem & $0.157\ [0.156,\,0.158]$ & $-1030.9\ [-1032.6,\,-1029.1]$ & $-0.223\ [-0.233,\,-0.214]$ & $0.498\ [0.497,\,0.499]$ & $<1.9\times 10^{-5}$ \\
\wikiwalk & $0.108\ [0.107,\,0.109]$ & $-4380.8\ [-4383.0,\,-4378.8]$ & $-0.261\ [-0.272,\,-0.250]$ & $0.461\ [0.460,\,0.462]$ & $<1.8\times 10^{-5}$ \\
\bottomrule
\end{tabular}
}
\end{table}

\subsection*{Shock adaptation and repeated error reduction}

The most important empirical distinction appears after shocks are introduced in Round~3. Table~\ref{tab:round-dynamics} reports round-wise quality and repeated error reduction. Only the full \inkh\ system improves materially from Round~1 to Round~4; all other baselines are flat or regress after shock introduction.

\begin{table}[t]
\centering
\footnotesize
\caption{Round-by-round quality dynamics and repeated error reduction.}
\label{tab:round-dynamics}
\begin{tabular}{lccccc}
\toprule
Baseline & $Q_1$ & $Q_2$ & $Q_3$ & $Q_4$ & Repeated error reduction \\
\midrule
\inkh      & 0.780 & 0.808 & 0.824 & \textbf{0.847} & \textbf{0.307} \\
\khnoinv   & 0.760 & 0.787 & 0.755 & 0.758 & -0.008 \\
\simplemem & 0.649 & 0.675 & 0.650 & 0.656 & 0.022 \\
\wikiwalk  & 0.701 & 0.729 & 0.697 & 0.700 & -0.005 \\
\bottomrule
\end{tabular}
\end{table}

Figure~\ref{fig:stale-rounds} makes the mechanism visible: after shocks, stale-memory usage spikes in all memory baselines except the full invalidation-enabled system.

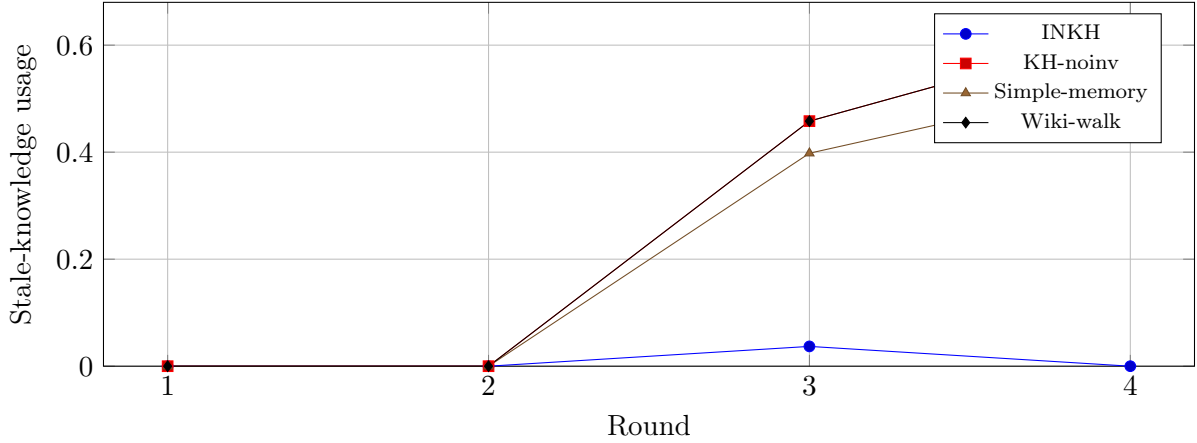
\begin{figure}[t]
\centering
\begin{tikzpicture}
\begin{axis}[
    width=0.97\linewidth,
    height=6.4cm,
    xlabel={Round},
    ylabel={Stale-knowledge usage},
    xmin=0.8, xmax=4.2,
    ymin=0, ymax=0.68,
    xtick={1,2,3,4},
    grid=major,
    legend style={at={(0.97,0.97)},anchor=north east,font=\scriptsize}
]
\addplot+[mark=*] coordinates {(1,0.000) (2,0.000) (3,0.037) (4,0.000)};
\addplot+[mark=square*] coordinates {(1,0.000) (2,0.000) (3,0.458) (4,0.624)};
\addplot+[mark=triangle*] coordinates {(1,0.000) (2,0.000) (3,0.398) (4,0.532)};
\addplot+[mark=diamond*] coordinates {(1,0.000) (2,0.000) (3,0.458) (4,0.624)};
\legend{INKH, KH-noinv, Simple-memory, Wiki-walk}
\end{axis}
\end{tikzpicture}
\caption{Stale-knowledge usage by round. Write-time invalidation is the key differentiator after shocks.}
\label{fig:stale-rounds}
\end{figure}

\subsection*{Task-family results}

Table~\ref{tab:tasks} shows quality by task family. \inkh\ is strongest across all four families and the largest gains are on the most operationally sensitive tasks: copy-trading evaluation and trade preparation.

\begin{table}[t]
\centering
\footnotesize
\caption{Task-family quality by baseline.}
\label{tab:tasks}
\resizebox{\linewidth}{!}{
\begin{tabular}{lrrrrrr}
\toprule
Task family & \inkh & \khnoinv & \wikiwalk & \simplemem & \toolagent & \modelonly \\
\midrule
Copy-trading evaluation & \textbf{0.819} & 0.753 & 0.693 & 0.657 & 0.537 & 0.437 \\
Market analysis & \textbf{0.793} & 0.748 & 0.693 & 0.665 & 0.556 & 0.456 \\
Portfolio review & \textbf{0.819} & 0.776 & 0.718 & 0.675 & 0.538 & 0.437 \\
Trade preparation & \textbf{0.827} & 0.783 & 0.723 & 0.633 & 0.531 & 0.431 \\
\bottomrule
\end{tabular}
}
\end{table}

On high-risk workflows only (copy-trading and trade preparation plus shock-tagged episodes), \inkh\ reaches quality $0.822$, stale-memory usage $0.018$, and traceability $0.999$, compared with $0.766$, $0.336$, and $0.923$ for \khnoinv, respectively. This is exactly where maturity and invalidation should matter most.

\subsection*{Governance ablation}

A possible alternative explanation for the gains is simply that \inkh\ stores more memory. Table~\ref{tab:inventory} rules that out. Both \inkh\ and \khnoinv\ accumulate essentially the same amount of knowledge and the same maturity mass. The difference is that \inkh\ invalidates obsolete memory rather than letting it persist.

\begin{table}[t]
\centering
\footnotesize
\caption{Knowledge inventory over 24 seeds. Both KH variants ingest the same amount of knowledge; only \inkh\ invalidates obsolete items.}
\label{tab:inventory}
\begin{tabular}{lccccc}
\toprule
Baseline & Final items & New items added & Invalidated items & Verified items & Proven items \\
\midrule
\khnoinv & 13.96 & 5.96 & 0.00 & 6.08 & 7.88 \\
\inkh    & 13.96 & 5.96 & 2.96 & 6.08 & 7.88 \\
\bottomrule
\end{tabular}
\end{table}

This result is important for product design. It implies that better financial cognition does not come merely from remembering more. It comes from \emph{remembering under governance}.

\section{Discussion and Limitations}

\subsection*{What the results imply for real products}

Four product lessons follow directly from the experiments.

\textbf{Passive injection should replace wiki walking in the foreground path.}
The agent-driven wiki baseline performs materially worse on both latency and token cost. In chat-like financial workflows, the system must assemble context before the model reasons, not ask the model to perform a multi-step search over its own memory.

\textbf{The graph should serve retrieval; the wiki should serve audit and review.}
This division preserves human interpretability without paying the full online cost of document-style traversal.

\textbf{Invalidation matters more than additional memory volume.}
Table~\ref{tab:inventory} shows that the performance gap between \inkh\ and \khnoinv\ arises from invalidation, not from more stored items. This is especially relevant in finance, where outdated assumptions can remain superficially plausible after regime breaks.

\textbf{The cognition plane and execution plane should be designed separately but coherently.}
Borjigin et al.\ show that execution safety must be enforced where side effects occur \cite{borjigin2026execution,borjigin2025safe}. Our results suggest an upstream complement: action survivability should be paired with cognition survivability. The user should not need to manually police the agent's memory state, any more than they should manually police its last-mile execution permissions.

\subsection*{Limitations}

This paper should be read with four limitations in mind.

First, the reported evaluation is a \emph{controlled synthetic benchmark}. It is designed to isolate architecture-level properties---latency, memory invalidation, reuse, and traceability---not live trading profitability. Second, the quality metric is simulator-defined rather than human-labeled. Third, the current artifact abstractly simulates graph-backed retrieval and serving behavior rather than instantiating a full production graph database. Fourth, public-data replay using FRED, EDGAR, and Binance interfaces is specified but not yet reported.

Accordingly, the right claim is not that \inkh\ proves financial alpha. The right claim is narrower and more architectural: \emph{interaction-native, governed knowledge harnesses are a better systems target for financial agents than turn-based retrieval plus disposable context}.

\section{Conclusion}

This paper introduced the interaction-native knowledge harness for continuous financial cognition. The central idea is simple: a financial agent should not force the user to manage the system's cognitive complexity. Instead, the system should absorb that complexity by continuously maintaining structured state, injecting the right context at the right time, and transforming interaction traces into governed long-term knowledge.

In a reproducible benchmark, this design improves quality, lowers latency relative to nontrivial memory baselines, sharply reduces stale-memory usage, and substantially increases decision traceability. The strongest result is not merely that \inkh\ remembers more. It is that \inkh\ remembers \emph{under governance}.

The broader implication is practical. If future financial agents are to be adopted in real workflows, they will need both continuous cognition and survivable execution. This paper addresses the first requirement. Recent survivability-aware execution work addresses the second \cite{borjigin2026execution,borjigin2025safe}. Together they suggest a more complete research program for financial AI: \emph{interaction-native cognition upstream, survivability-aware execution downstream}.

\appendix

\section{Practical Engineering Appendix}

\subsection{Implementation defaults}

Table~\ref{tab:defaults} separates \emph{required architectural commitments} from \emph{suggested defaults}. The latter are implementation recommendations rather than claims of optimality.

\begin{table}[H]
\centering
\footnotesize
\caption{Implementation defaults. ``Required'' means necessary for the claimed architecture; ``Suggested'' means recommended starting value.}
\label{tab:defaults}
\begin{tabular}{p{5.0cm} p{1.5cm} p{2.0cm} p{5.3cm}}
\toprule
Setting & Status & Default & Note \\
\midrule
Passive foreground injection & Required & On & Central architectural feature \\
Working context buffer budget & Required & 300--600 tokens & Experiment default: 600 \\
Temporal graph memory & Required & Enabled & Canonical entities + temporal edges \\
Wiki audit surface & Required & Enabled & Human-readable review layer \\
Write-time invalidation & Required & Enabled in \inkh & Key governance feature \\
Maturity gating & Required & Enabled & Strictest for high-risk workflows \\
Embedding dimension & Suggested & 1024 & Vendor/model-agnostic default \\
Structural candidate pool & Suggested & 100--300 edges & Before reranking and compression \\
Hop radius & Suggested & 1--2 hops & Entity-neighborhood retrieval \\
Top returned items & Suggested & 12--20 compact items & Before template compression \\
Simple-memory budget & Suggested & 1200 tokens & Experiment baseline \\
Wiki-walk budget & Suggested & 4200 tokens & Experiment baseline \\
Maintenance probe mix & Suggested & 30/20/15/15/10/10 & Hub/link/quality/stale/merge/split \\
\bottomrule
\end{tabular}
\end{table}

\subsection{Synthetic benchmark mechanics}

The released artifact implements the benchmark as a deterministic simulator with seed-controlled randomness. The quality function is:
\[
Q
=
\mathrm{clip}\!\left(
q_b
+
\beta_r(r-1)
+
\beta_h \cdot \mathrm{hits}
-
\beta_m \cdot \mathrm{missing}
-
\beta_s \cdot \mathrm{stale}
+
\epsilon
\right),
\]
where $q_b$ is a baseline-specific prior, $r$ is the round index, $\epsilon \sim \mathcal{N}(0,\sigma^2)$ with $\sigma=0.018$, and $\beta_h,\beta_s$ vary by baseline. For the KH variants, the retrieval-hit bonus is capped at $0.16$ with coefficient $0.028$. For \wikiwalk\ it is capped at $0.13$ with coefficient $0.024$. For \simplemem\ it is capped at $0.12$ with coefficient $0.022$. Missing gold requirements incur a penalty of $0.02$ per miss. Stale-memory use incurs a penalty of $0.11$ for \wikiwalk, \simplemem, and \khnoinv, and $0.04$ for the full \inkh.

\subsection{Data schemas}

A minimal raw evidence record:
\begin{lstlisting}
{
  "chunk_id": "raw_000123",
  "source_type": "market_api",
  "source_ref": "binance:BTCUSDT:1h",
  "timestamp": "2026-05-13T08:00:00Z",
  "trust_tier": "high",
  "content": "... immutable source payload ..."
}
\end{lstlisting}

A minimal entity record:
\begin{lstlisting}
{
  "entity_id": "asset:BTC",
  "canonical_name": "BTC",
  "aliases": ["Bitcoin", "XBT", "BTCUSDT"],
  "entity_type": "ASSET",
  "summary": "Primary crypto asset used as liquidity anchor.",
  "updated_at": "2026-05-13T08:05:00Z"
}
\end{lstlisting}

A minimal edge record:
\begin{lstlisting}
{
  "edge_id": "edge_004512",
  "src": "asset:BTC",
  "dst": "risk:slippage",
  "relation_type": "affected_by",
  "description": "Observed slippage rises under high volatility.",
  "evidence_ids": ["raw_000123", "trace_000045"],
  "confidence": 0.83,
  "maturity": "verified",
  "regime_tag": "high_volatility",
  "valid_at": "2026-05-10T00:00:00Z",
  "invalid_at": null,
  "updated_at": "2026-05-13T08:05:00Z"
}
\end{lstlisting}

\subsection{Reproduction}

The released artifact contains configuration files, simulator code, synthetic data, result tables, and figure-generation scripts. To reproduce the synthetic benchmark, run:

\begin{verbatim}
python scripts/run_synthetic_suite.py
\end{verbatim}

\noindent This regenerates per-workflow logs, summary tables, confidence intervals, paired tests, and result figures.

\end{document}